\documentclass[sigconf,noamssymb]{acmart}

\usepackage{amsmath, amssymb}

\usepackage{float}
\floatstyle{ruled}
\newfloat{procedure}{htbp}{lop}
\floatname{procedure}{Procedure}
\AtBeginDocument{%
  \providecommand\BibTeX{{%
    \normalfont B\kern-0.5em{\scshape i\kern-0.25em b}\kern-0.8em\TeX}}}

\setcopyright{acmlicensed}
\copyrightyear{2025}
\acmYear{2025}
\acmDOI{XXXXXXX.XXXXXXX}

\acmConference[CODE@MIT]{}{Nov 14-15,
  2025}{MIT}
\acmBooktitle{2024 Conference on Digital Experimentation @ MIT (CODE @ MIT)}
\acmISBN{978-1-4503-XXXX-X/18/06}

\usepackage{graphicx} % Required for inserting images
\usepackage{amsmath, amssymb, algorithm, algpseudocode}
\usepackage{caption}
\usepackage{algorithm}
\usepackage{algpseudocode}
\usepackage{multicol}
\usepackage{graphicx}
\usepackage{tikz}
\usepackage{subcaption}
\usetikzlibrary{shapes.geometric, arrows}

\tikzstyle{startstop} = [rectangle, rounded corners, minimum width=3cm, minimum height=1cm,text centered, draw=black, fill=red!30]
\tikzstyle{process} = [rectangle, minimum width=3cm, minimum height=1cm, text centered, draw=black, fill=orange!30]
\tikzstyle{decision} = [diamond, minimum width=3cm, minimum height=1cm, text centered, draw=black, fill=green!30]
\tikzstyle{arrow} = [thick,->,>=stealth]

\title{RIE-Greedy: Regularization‑Induced Exploration for Contextual Bandits}

\author{Tong Li}
\affiliation{
  \institution{University of Toronto}
  \country{Canada}
}
\author{Thiago de Queiroz Casanova}
\affiliation{
  \institution{Braze}
  \country{United States}
}

\author{Eric M. Schwartz}
\affiliation{
  \institution{University of Michigan}
  \country{United States}
}

\author{Victor Kostyuk}
\affiliation{
  \institution{Braze}
  \country{United States}
}

\author{Dehan Kong}
\affiliation{
  \institution{University of Toronto}
  \country{Canada}
}

\author{Joseph J. Williams}
\affiliation{
  \institution{University of Toronto}
  \country{Canada}
}

\date{September 2025}

\begin{document}

\begin{abstract}

Real-world contextual bandit problems with complex reward models are often tackled with iteratively trained models, such as boosting trees. However, it is difficult to directly apply simple and effective exploration strategies, such as Thompson Sampling or UCB, on top of those black-box estimators. Existing approaches rely on sophisticated assumptions or intractable procedures that are hard to verify and implement in practice. In this work, we explore the use of an exploration-free (pure-greedy) action selection strategy, that exploits the randomness inherent in model fitting process as a intrinsic source of exploration. More specifically, we note that the stochasticity in cross-validation based regularization process can naturally induce Thompson Sampling–like exploration. We show that this regularization-induced exploration is theoretically equivalent to Thompson Sampling in the two-armed bandit case and empirically leads to reliable exploration in large-scale business environments compared to benchmark methods such as $\varepsilon$-greedy and other state-of-the-art approaches. Overall, our work reveals how regularized estimator training itself can induce effective exploration, offering both theoretical insight and practical guidance for contextual bandit design.

%TODO: add the 'what is good about iteration based learnier' to intro.

% Real-world contextual bandit problems often involve complex or unknown reward structures. In such settings, iteration-based learners such as neural networks and boosting trees are widely used as flexible approximators to capture underlying patterns. However, due to their lack of theoretical guarantees, it is difficult to directly apply common exploration strategies—such as Thompson Sampling or UCB—on top of their fitted estimators. Existing methods that are compatible with these learners often rely on assumptions or procedures that are hard to implement and verify in practice. We XXXX (connect to what we have before)

\end{abstract}

\maketitle
\section{Introduction}
Personalized sequential decision-making~\citep{murphy2003optimal}, where the goal is to select the most effective action for each individual based on their evolving state, behavior, and contextual information, is a common and complex challenge in domains such as digital marketing\citep{codis2019decision}, personalized recommendations~\citep{li2010contextual}, and healthcare~\citep{tewari2017mobile,tewari2021precision}.For instance, in digital marketing, when sending a promotional email one must decide on the best frequency, timing, offer, text, and creative to send to each customer based on their profile, the environment around them, and past behavior. The defining challenge is to make a decision when we only observe the outcomes of the set of decisions that was taken, without knowing ``what would have happened'' had alternative decisions been made. One efficient approach to address this partial-information feedback---referred to as \emph{bandit feedback}---requires balancing exploration (trying less-certain options to learn more about reality) and exploitation (choosing the currently best-known option to maximize reward). And in the presence of other variables describing those options and the context of each decision, this challenge is formalized by the \emph{contextual bandit} framework\citep{langford2007epoch, lattimore2020bandit}.

In practice, contextual bandit problems often involve a large number of features, and the underlying reward function is typically complex and not well captured by well-studied parametric approaches such as linear models or generalized linear models (GLMs). Practitioners therefore rely on flexible learners, such as boosting trees, ensemble of trees, or neural networks, to iteratively approximate the underlying reward function. These models are trained using standard machine learning routines--splitting data into training and validation sets, performing cross-validation, applying early stopping when training iterative models, and tuning hyperparameters. 
In bandit literature, such training routine are often referred to as offline regression oracle~\citep{foster2018practical}, where the 'oracle' knows how to iteratively train the model and reduce the loss function (such as least square loss). 
Because these procedures are widely taught in introductory machine learning courses, they have become the de facto standard in many industry applications and are often adopted as the first step when approaching contextual bandit problems.

%TODO: refer this to offline regreesion oracle
%better explain oracle, adn don't focus too much on it.

However, the downside of such flexibility is that these methods often lack well-defined statistical properties, such as closed-form variance estimates. As a result, it is not straightforward to directly implement classical exploration strategies such as Thompson Sampling (TS)\citep{russo2018tutorial} on top of these complex or even black-box models. At this stage, practitioners generally have two possible directions to proceed after implementing the reward estimator. They can either try to find a contextual bandit algorithm to apply on top of their trained reward function, or they can naively use the trained function to select the best estimated action (i.e., a purely greedy approach). In this work, we show that such a naive greedy strategy can perform very competitively. Before presenting our main results, we first take a brief detour to provide a broader overview of more theoretical developments in this space.

First, there is a broad line of work that seeks to design general-purpose, model-free contextual bandit algorithms. Among these, the most relevant and practical branch is the aforementioned class of offline regression-oracle–based methods~\citep{foster2018practical}. These approaches delegate the reward estimation process described above to an offline regression \emph{oracle} and apply an explicit exploration strategy on top of the oracle’s output. Notably, in these kinds of approaches, when new data is available, the model cannot be updated without re-estimating the entire model on its prior data and the newest data. Unlike in "online" learning, which is far more efficient at incorportaing new data into a model, this traditional "offline" learning or batch-based learning suffers from large computational costs that grow substantially as new data accumulates. But some methods have designed ways to reduce computation cost in terms of the frequency of oracle call needed. A representative example is FALCON \citep{SimchiLeviXu2022}. It computes allocation probabilities in a relatively simple form and significantly reduces the number of oracle calls compared to earlier methods, while achieving theoretically optimal regret guarantees. Nevertheless, it still involves parameters or assumptions that are difficult to specify and validate in practice. For example, to get the action allocation probability, practitioners must obtain the time-varying bounds on the oracle’s least-squares error under an arbitrary data collection process and contextual distribution. If the bound is chosen too big/small, the algorithm empirically can degenerate trivially to pure greedy or pure uniform randomization.
%Practitioners must not only estimate these bounds numerically but also define them as time-varying sequences as data accumulate. 
Moreover, for theoretical tractability, these algorithms often require less practical procedure. For example, Falcon+ requires discarding all previously collected data and refitting the estimator using only the most recent epoch’s data. Such a rigid training schedule may not be efficient in practical stationary case, and it doesn't align well with practical deployment scenarios, where models are typically updated incrementally or through sliding time windows to accommodate non-stationarity. Finally, despite their elegant theoretical foundations, these methods have received limited empirical evaluation in realistic, large-scale settings.

Interestingly, in contrast to these theoretically refined algorithms, accumulating evidence shows that purely greedy action selection can also achieve strong regret performance—both theoretically~\citep{bastani2021mostly} and empirically~\citep{bietti2021bakeoff}—for example, in linear reward settings when the number of contextual features is large. This offers a new understanding of the problem: it suggests that contextual bandits may contain natural sources of exploration, thereby diminishing the necessity and centrality of explicit exploration strategies.

%TODO : add 
This finding is also practically attractive: a purely greedy approach allows practitioners to focus solely on improving the reward model itself, greatly simplifying the deployment pipeline. In practice, this can save substantial research and development effort by avoiding the need to fine-tune exploration hyperparameters or to understand and validate the theoretical assumptions underlying complex bandit algorithms. Moreover, its simplicity makes it flexible and compatible with additional procedures required by practical needs. For example, in contrast to the rigidity of FALCON, practitioners can easily incorporate a sliding time window to adapt it to non-stationary cases.

%because its only the estimator. the overall complexity of the approach makes it easier to fit into
%orgna
%reduces teh barrier of adaoption than many other s 
%because it not adding a new step (a few more assumption and resitrcitino )

%relying on just the estimator is that the estimator is more responsive to changes in the data than estimator and explretion 

However, as mentioned earlier, the existing foundings remains limited in scope. Existing analyses and promising results for exploration-free approaches are confined to linear reward functions and context-rich (i.e., the variety of context is huge), stationary environments. This is because: in these settings, the diversity of contextual feature combinations can induce \emph{passive exploration}. As a result, purely greedy selection can achieve adequate coverage of the action space and thereby obtain sufficient exploration. Nevertheless, real-world applications often involve complex, non-linear, and potentially non-stationary reward functions. Notably, prior work explicitly suggests that such passive exploration may fail in non-stationary cases. Moreover, it is also empirically difficult to verify whether contextual diversity is sufficient. 
%In fact, these challenges fundamentally limit the general applicability of pure-greedy methods.

\subsection{Our Contribution}
In this work, we extend the investigation and use of exploration-free strategy to non-stationary, non-linear, and arbitrarily rich (or even low/no) context settings. We realize that, beside the known phenomenon of \emph{passive exploration} induced in greedy algorithms by rich contextual diversity, the estimation process itself can \emph{intrinsically} introduce an additional layer of exploration. To the best of our knowledge, such links between estimator fitting and bandit exploration have been overlooked in previous work. We focus our study on iterative learners such as boosting trees in real business scenarios. We show that, when trained under regularization, the reward estimator itself can introduce exploration, which further reduces the need for doing additional exploration in practice.

%the e p induces the exploraton because of the early stopping intrinsic to the method.
%This property enables purely greedy selection to remain effective even when contextual variation is limited or when the reward distribution drifts over time.
More specifically, when fitting a complex model iteratively and using cross-validation on randomly split validation sets to determine the early-stopping iteration, the randomness in the validation split introduces stochasticity into the training process. In Section~\ref{sec:MAB_reward_analysis}, we demonstrate that, when simplified to a two-armed multi-armed bandit setting (i.e., no contextual features), this regularization process of early stopping, originally designed for preventing over-fitting, behaves analogously to a hypothesis-testing procedure: each additional iteration is rejected with a probability similar to a $p$-value threshold (again, the randomness arises from the validation split). Conceptually, the learner continues training in proportion to the likelihood that the learned reward pattern truly exists in the validation data. This mechanism parallels the principle of Thompson Sampling, which selects actions in proportion to their likelihood of being optimal, reflecting the intrinsic uncertainty in beliefs about their rewards. We refer to this phenomenon as \emph{regularization-induced exploration}. Exploration is intrinsic to early-stopping in cross-validation for iterative learners.

Empirically, as shown in our results section, this intrinsic exploration mechanism makes pure-greedy strategy with early stopping perform almost identical to Thompson Sampling in the two-armed case. Moreover, we evaluate our approach on both complex stationary and non-stationary cases generated from real large-scale business environments. It shows strong performance compared to popular approaches such as $\epsilon$-greedy and theoretically optimal algorithm such as FALCON.

% Lastly, this exploratory behavior--emerging naturally from the model-fitting process--has been largely overlooked in prior work, where estimator training is treated purely as a loss-minimization step and external exploration is imposed afterward. We are the first to discuss the intrinsic link between loss-reduction under regularization via early stopping and bandit exploration.

In summary, our work explores the use of intrinsic exploration behavior arising from the regularized estimator training process. Theoretically, it offers new insights into how learning dynamics themselves can drive exploration in contextual bandits. Practically, it provides guidance that can either remove the need to design explicit exploration strategies or help practitioners better estimate how much exploration is needed (potentially less than previously expected) and more easily scope the parameter range.

\section{Problem Setup}
Our problem setting is closely related to the regression-oracle-based contextual bandit setting \citep{foster2018practical}, while also having important distinctions. In this section, we first give the basic contextual bandit setup, then describe the offline regression-oracle-based setting, and finally highlight our distinctions.

We consider a standard contextual bandit problem where an action/arm is denoted as $a$ and a context is denoted as $x$. Let $\mathcal{A}$ be the set of all available actions, with a total number of $K$ distinct actions. Further, let $\mathcal{X}$ be the context space and $\mathcal{D}_{\mathcal{X}}$ be the distribution of the context.

At each round $t = 1, \dots, T$, a context $x_t \sim \mathcal{D}_{\mathcal{X}}$ is sampled by nature, and corresponds to a reward vector $\vec{r}_t \in [0,1]^K$. The agent selects an action $a_t \in \mathcal{A}$ according to a policy/algorithm $\pi$, and receives the reward $r_t(a_t)$, which is partial information from the full reward vector $\vec{r}_t$. Note that the reward vector $\vec{r}_t$ depends on both the context $x_t$ and the action $a_t$.

Let $H_t = (x_1, a_1, r_1, \dots, x_t, a_t, r_t)$ be the ordered tuple of history. The bandit algorithm $\pi$ selects actions by mapping history and the current context $(H_{t-1}, x_t)$ to an arm-selection distribution $P_{\mathcal{A},t}$.

%We consider actions that are represented by feature vectors. That is, different actions can share some common feature values, so that learning about one action may help infer the rewards of others. 

The objective of the agent is to maximize the cumulative expected reward:
\[
\max_{\pi}\; \sum_{t=1}^T \mathbb{E}[\,r_t(a_t, x_t)\,].
\]

\subsection{Offline Regression Oracle Based Setting}
A practical approach is to reduce the complex contextual bandit problem to an offline regression problem. The broader idea is: at each step, we train a reward estimator using historical data and consider action selection only based on the trained estimator—ignoring additional information from the dataset itself. \\

This is formed as a offline regression-oracle–based problem \citep{SimchiLeviXu2022}. In this setting, we assume that the learning and estimation of the reward function are performed by training an estimator, such as a boosting tree, using a squared-loss objective.
Specifically, let $\mathcal{F}$ denote a class of predictors $f: \mathcal{X} \times \mathcal{A} \to [0,1]$. The classic regression-oracle setting assumes access to an oracle that, given a dataset $H = \{(x_i, a_i, r_i)\}_{i=1}^n$, returns
\[
\text{ORACLE}(H) = \arg\min_{f \in \mathcal{F}} \sum_{(x,a,r) \in H} (f(x,a) - r)^2.
\]
Note that the oracle does not have access to the true reward function. Instead, it represents a process that maps historical data—contexts, actions, and rewards—into a predictor that minimizes the squared loss. In practice, this corresponds to the estimator training (fitting) process.

The above oracle definition is a simplification of the training process in reality. For example, for complex models such as boosting trees, people often need to perform regularization and hence do not exactly minimize the loss on the training set.

\paragraph{Our Key Distinction.}
However, the classical formulation represents an idealized and simplified scenario. In practice, particularly beyond linear or generalized linear models, practitioners often need to prevent overfitting rather than exactly minimize the squared error loss on the training data $H$. This suggests that the practical implementation of an algorithm can be different from its theoretical formulation.

While it seems like the simplification of the estimation method is a minor one, we argue that considering alternative estimation methods, such as regularization and early stopping, actually has great benefits for contextual bandit algorithms.

In our problem formulation, we discuss the impact of the estimator training process on the resulting action selection pattern, exploration behavior, and reward performance. Specifically, we compare estimators trained with a fixed iteration to estimators trained using standard early stopping via cross-validation. We also compare additional explicit exploration strategies applied on top of the trained estimator. Finally, we evaluate both stationary and non-stationary settings, where in the latter case the reward distribution may drift over time and the agent operates with a sliding history window.

% \section{Problem Formulation}
% We consider a standard contextual bandit problem with finite action set $\mathcal{A}$ of size $K$, context space $\mathcal{X}$, and bounded stochastic rewards $r(a,x)\in[0,1]$. At each round $t=1,\dots,T$, a context $x_t\sim \mathcal{D}_{\mathcal{X}}$ is observed, the agent selects an action $a_t\in\mathcal{A}$ according to a policy $\pi_t$ that maps $(H_{t-1},x_t)$ to a distribution $P_{\mathcal{A},t}$ over $\mathcal{A}$, where $H_{t-1}$ denotes history up to time $t-1$, and receives reward $r_t \sim P(R\mid a_t,x_t)$. Actions are represented by feature vectors. In other words, different actions can have some shared action feature values, so that learning about one action can potentially help with informing others. The objective is to maximize cumulative reward:
% \[
% \max_{\pi}\;\sum_{t=1}^T \mathbb{E}[\,r_t(a_t,x_t)\,].
% \]

% Finally, We evaluate both stationary and non-stationary settings, where in the latter case the reward distribution may drift over time and the agent operates with a sliding history window.

\section{Related Work}

%[Not sure where to put this: But there is a chapter and a related NIPS paper out there about "perturbation" methods in estimation resembling Thompson Sampling. Chapter https://www.ambujtewari.com/research/abernethy16perturbation.pdf and NIPS 2019  https://www.ambujtewari.com/research/kim19optimality.pdf ]

The contextual bandit framework provides a principled way to balance exploration and exploitation in sequential decision-making. Classical algorithms such as UCB-based methods \citep{auer2002finite,abbasi2011improved} and Thompson Sampling \citep{russo2018tutorial, agrawal2013thompson} achieve strong regret guarantees under realizability assumptions, such as linearity or bounded function classes. However, these assumptions are often unrealistic in practical applications where the reward function is nonlinear, high-dimensional, and varies across contexts.

A major research direction in the \emph{agnostic} setting aims to design general-purpose and computationally tractable algorithms that perform well for arbitrary policy classes without relying on strong assumptions. Early approaches such as Regressor Elimination \citep{agarwal2012contextual}, Policy Elimination \citep{dudik2011efficient},  and ILOVETOCONBANDITS \citep{agarwal2014taming} introduced the use of a \emph{cost-sensitive classification oracle}. This idea enabled a non-trivial “online-to-offline reduction,” translating the contextual bandit problem into repeated calls to an offline supervised learner, and it can work with . While these algorithms achieve optimal theoretical regret, they are computationally intractable in practice: the cost-sensitive classification oracle itself is NP-hard to solve \citep{agrawal2016linearcbwk}, and the algorithms typically require repeated dataset augmentation, leading to excessive memory and runtime demands \citep{SimchiLeviXu2022}.  

Subsequent work sought to improve scalability by replacing the classification oracle with a \emph{regression oracle}\citet{foster2018practical, foster2020beyond, zhu2022contextual}. Those approaches leverage a least-squares regression oracle to approximate the reward function, significantly improving computational efficiency and aligning more closely with common machine learning practice. More recently, \citet{SimchiLeviXu2022} introduced the FALCON algorithm, which achieves theoretical optimality for general function classes using an \emph{offline regression oracle}. FALCON represents a key step forward: it replaces cost-sensitive classification with standard regression, thereby broadening applicability to infinite or nonparametric function classes (e.g., boosting trees, neural networks) while remaining oracle-efficient. Nonetheless, its practical deployment is limited by assumptions that are difficult to verify, such as requiring explicit bounds on the oracle’s expected squared error under arbitrary data distributions, and by structural dependencies (e.g., increasing epoch schedules) that are not naturally compatible with non-stationary or sliding-window learning. Meanwhile, \citet{zhao2025sharp} is a concurrent work to ours that focuses on KL-regularized regret and Reinforcement Learning from Human Feedback (RLHF) settings. Their approach can work with a regression oracle and simplifies to the EXP algorithm in contextual bandit reward settings. It has similar reward guarantees and practical limitations as FALCON, and we evaluate it together with FALCON in our simulation study.

% Among existing approaches, FALCON is particularly relevant to our study. It extends the regression-oracle line of work and has been shown to perform favorably against a broad range of contextual bandit algorithms evaluated in the \emph{Contextual Bandit Bakeoff} \citep{foster2018practical}. Its combination of strong theoretical guarantees and practical use of offline least-squares regression makes it a promising baseline for real-world applications. However, as discussed above, its implementation still depends on several parameters that are difficult to specify, and its epoch-based structure does not naturally accommodate non-stationary environments.

Our work takes a different perspective. Rather than treating the oracle solely as a prediction module and designing exploration on top of it, we highlight that the oracle training process itself can be leveraged as an intrinsic source of exploration. In modern contextual bandit systems, reward estimators are often implemented using flexible, iteration-based learners such as gradient-boosted trees or neural networks, trained with standard regularization and early-stopping procedures. These mechanisms inherently introduce stochasticity in model fitting through random validation splits and iteration-dependent variance, which can naturally induce exploration behavior. By leveraging this intrinsic property, we unify estimation and exploration within a single layer, leading to an approach that is simple, interpretable, and computationally efficient, while remaining directly compatible with large-scale machine learning pipelines and adaptable to non-stationary environments. This perspective connects standard machine learning regularization techniques with bandit exploration, offering a novel and practical direction for designing general-purpose contextual bandit algorithms.

\section{Method}
In this work, we investigate the intrinsic exploration behavior of the common machine learning regularization approach—early stopping. In other words, we do not take additional exploration strategies and simply select actions following the pure-greedy strategy on top of it. We illustrate this common procedure in Procedure~\ref{proc:early-stopping}. Note that for line 3 in the procedure (finding the best step-wise base learner), it takes different forms for different iterative learners, and we take the boosting tree format as an example. Also, in practice such base learner is fitted using approximation and involves many tuning parameters (e.g., for boosting trees, it depends on the split on training/validation sets, the minimum number of points in a leaf, the depth of the tree). We delegate the training parameter tuning process to the business routine (i.e., they are not tuning parameters introduced by our approach), and this is commonly assumed in any work with regression oracle.

Our contribution is to demonstrate that the variability in the estimator’s decisions, arising from early stopping and cross-validation, naturally induces exploration while preserving exploitation. This can greatly reduce the need for designing additional exploration strategies. This perspective differs from prior work, which typically assumes an oracle estimator and then adds exploration mechanisms on top. Our approach instead makes use of what is already embedded in the training process, requiring no additional components beyond the base learner.

At a high level, the procedure evaluates each boosting step on a held-out cross-validation set and continues to the next iteration of training only if the new learner improves predictive performance. The likelihood of stopping is proportional to the uncertainty in the data. Stopping at earlier iterations effectively allows the model to learn less and explore more. As we show in Section~\ref{sec:MAB_reward_analysis}, this mechanism yields allocation probabilities that align closely with Thompson Sampling in the two-armed bandit case. Further, in Section~\ref{sec:result} we show that empirically, in real business-inspired simulation environments, such exploration can lead to desirable reward in both stationary and non-stationary settings. Our work highlights that the regularization procedure in estimator fitting can act as an exploration strategy, and hence practitioners can downplay or remove the need for researching additional explicit exploration strategies. In cases where practitioners would like to apply additional exploration strategy, our work entails that they are more likely to set the algorithm parameter for the exploration strategy with less amount of exploration and more toward the exploitation end.

% The algorithm has one tuning parameter, $\alpha$, which controls the degree of exploitation. When context features are sufficiently rich, explicit exploration is less critical, and the model can rely 
% more on exploitation to maximize reward. This aligns with  \citep{bastani2021mostly}, who show that greedy policies can succeed with little or no explicit exploration under such conditions. In practice, we set $ \alpha = 1.5 $.
% The other inputs, $\tau$ and $n_{\text{wait}}$, are determined by the application setting rather than by the algorithm itself: 
% $\tau$ reflects the natural update frequency of the system (e.g., retraining once per day in a business application), and $n_{\text{wait}}$ is a standard training hyperparameter for early stopping.

In our simulation studies, we focus on applying our approach to boosting trees, a widely used base model in large-scale contextual bandit applications.  However, the idea is generalizable to other \emph{iteration-based learners} such as neural networks. The approach is lightweight, requires no changes to the base model, and integrates directly with standard ML pipelines.

% \begin{algorithm}[h]
% \caption{Greedy with Early Stopping}
% \label{alg:greedy-es}
% \begin{algorithmic}[1]
% \State \textbf{Input:} estimator regular training hyper-parameters $\theta$, epoch length $\tau$
% \For{epoch $m = 1,\ldots,T$}
%     \State Receive contexts $x_{m,1},\ldots,x_{m,\tau} \in \mathcal{X}$
%     \State Split history $H_{m-1}$ into training and validation sets ($H_{tr}, H_{val}$, 50/50 split)
%     \State Train base learner with early stopping on data sets $(H_{tr}, H_{val})$ under hyper-parameters $\theta$, obtain stopping iteration $n_{\text{stop}}$
%     \State Retrain on full history $H_t$ with selected iteration $n_{\text{stop}}$
%     \For{$i = 1,\ldots,\tau$} \Comment{perform greedy selection}
%         \State Predict reward estimates $\hat{r}(a \mid x_{m,i})$ for all $a \in \mathcal{A}$
%         \State Select greedy action $a_{m,i} \gets \arg\max_{a} \hat{r}(a \mid x_{m,i})$ and receive reward $r_{m,i}$.
%     \EndFor
%     \State Update history $H_m = (x_{1,1},a_{1,1},r_{1,1},...,x_{m,\tau},a_{m,\tau},r_{m,\tau})$
% \EndFor

% \end{algorithmic}
% \end{algorithm}

\begin{procedure}[h]
\caption{Iterative learner trained under early stopping (simple cross-validation)}
\label{proc:early-stopping}
\begin{algorithmic}[1]

\State \textbf{Input:} Training data 
$H_{\text{tr}} = \{(x_i,a_i,r_i)\}$\footnotemark[1], validation data $H_{\text{val}}$;
learning rate $\eta$; maximum rounds $M$; 
maximum waiting rounds $n_{\text{wait,max}}$

\Statex \textbf{Initialization:}
\Statex \hspace*{1.5em} $F_0(x) 
= \frac{1}{|H_{\text{tr}}|}\sum_{(x_i,a_i,r_i)\in H_{\text{tr}}} r_i$
\Statex \hspace*{1.5em} $L_{\text{val}}^{(0)}
= \frac{1}{|H_{\text{val}}|}
  \sum_{(x_i,a_i,r_i)\in H_{\text{val}}}
  (r_i - F_0(x_i))^2$
\Statex \hspace*{1.5em} $L_{val,best} \gets L_{\text{val}}^{(0)}$;
                        $F_{\text{best}} \gets F_0$;
                        $n_{\text{wait}} \gets 0$

\For{$m = 1$ to $M$}

    \State $f_m\footnotemark[2] = \arg\min_{f \in \mathcal{H}}
      \sum_{(x_i,a_i,r_i)\in H_{\text{tr}}}
      (r_i - (F_{m-1}(x_i) + f(x_i)))^2$

    \State $F_m(x) = F_{m-1}(x) + \eta\, f_m(x)$

    \State $L_{\text{val}}^{(m)}
    = \frac{1}{|H_{\text{val}}|}
      \sum_{(x_i,a_i,r_i)\in H_{\text{val}}}
      (r_i - F_m(x_i))^2$

    \If{$L_{\text{val}}^{(m)} < L_{val,best}$}
        \State $L_{val,best} \gets L_{\text{val}}^{(m)}$;
               $F_{\text{best}} \gets F_m$;
               $n_{\text{wait}} \gets 0$
    \Else
        \State $n_{\text{wait}} \gets n_{\text{wait}} + 1$
    \EndIf

    \If{$n_{\text{wait}} \ge n_{\text{wait,max}}$}
        \State \textbf{break}
    \EndIf
    
\EndFor

\State \Return $F_{\text{best}}$

\end{algorithmic}
\end{procedure}

\footnotetext[1]{%
Since the estimator is trained offline, we treat $H_{\text{tr}}$ (and $H_{\text{val}}$) as unordered datasets.
This differs from the ordered tuple $H_t$ used to denote contextual bandit interaction histories.}

\footnotetext[2]{%
In practice, fitting the base learner involves hyperparameters such as tree depth, minimum samples per leaf, and other estimator-specific training options.}

\section{Exploration Behavior Analysis in a Two-Armed Bandit Case}\label{sec:MAB_reward_analysis}

We illustrate how early stopping with cross-validation naturally induces exploration.  
The key insight is that early stopping functions as a hypothesis test: a new learner is accepted only when training and validation evidence aligns, and the probability of acceptance increases with the strength of the signal, analogous to a $p$-value decreasing as evidence grows.  
This parallels the philosophy of Thompson Sampling, which acts according to its belief under uncertainty.  

In this section, we establish that the proposed method achieves regret comparable to Thompson Sampling by showing that its allocation probabilities closely match those of Thompson Sampling in a simplified two-armed binary reward setting.  
We use a tree model for illustration, but the argument extends to other aggregated learners such as neural networks.

\subsection{Notation and Specialization}
\begin{figure*}[h]
    \centering
    \includegraphics[width=0.9\textwidth]{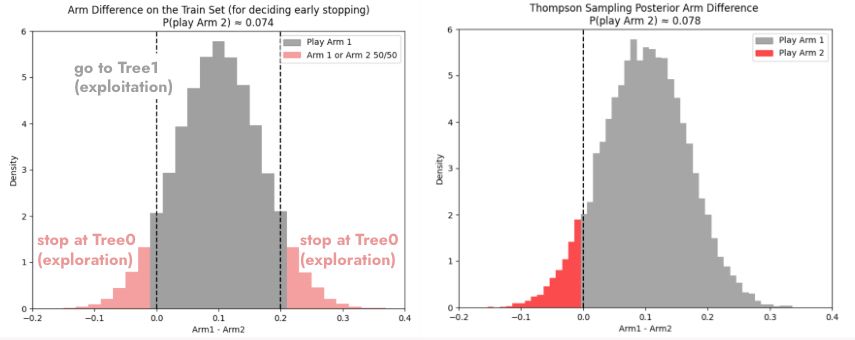}
    \caption{Allocation probability comparison between simplified two-step boosting Tree with early stopping (left) and Thompson Sampling (right), with $N_1=100$, $N_2=100$, and reward means $0.6$ and $0.5$.}
    \label{fig:ap_comparison}
\end{figure*}

To make the analysis precise in this simplified two-arm, tree-based setting, we introduce the following notation. 
Let $x \in \{1,2\}$ denote the arm indicator. More generally, $x$ can represent a higher-dimensional context including attributes of the arms. 
In this setting, a tree takes $x$ as input and outputs a prediction $\text{Tree}(x)$.  

We analyze the probability of arm assignment at an arbitrary but fixed time $t$. 
Given history $H_t$, let $N_i$ be the number of samples from arm $i$ up to time $t$, and let the empirical means on the full sample be $\bar r_1, \bar r_2$. 
Define their difference as $\Delta = \bar r_1 - \bar r_2$, and without loss of generality assume $\Delta > 0$.  

For simplicity, assume we split the history evenly into training and validation sets, $H_{tr}$ and $H_{val}$, each containing $n_i = N_i/2$ samples from arm $i$. 
Denote the training and validation reward means for arm $i$ by $\bar r_{tr,i}$ and $\bar r_{val,i}$, and the overall means by $\bar r_{tr}$ and $\bar r_{val}$.  
We also define the arm differences in the splits as $\Delta_{tr} = \bar r_{tr,1} - \bar r_{tr,2}$, $\Delta_{val} = \bar r_{val,1} - \bar r_{val,2}.$

These quantities allow us to express the predictions of Tree$_0$ and Tree$_1$, and later to derive equivalent conditions for early stopping.

\paragraph{Boosting procedure in the two-armed case.} 
Boosting begins with a degenerate tree (Tree$_0$) that predicts the overall training mean reward:
\[
\text{Tree}_0(x) = \bar{r}_{tr} = \frac{n_1 \bar{r}_{tr,1} + n_2 \bar{r}_{tr,2}}{n_1 + n_2}.
\]
The next tree (Tree$_1$) is guaranteed to split on the arm indicator and predicts deviations from $\bar{r}_{tr}$:
\[
\text{Tree}_1(x) = 
\begin{cases}
\bar{r}_{tr,1} - \bar{r}_{tr}, & x=1, \\
\bar{r}_{tr,2} - \bar{r}_{tr}, & x=2.
\end{cases}
\]

%can improve later
Thus, the mechanism can be viewed as combining early stopping with greedy action selection: if training stops at Tree$_0$, the algorithm explores uniformly, while if it continues to Tree$_1$, it exploits greedily. The decision of whether to stop, and its likelihood given the data, introduces the effective exploration probability.

\paragraph{Early stopping with cross-validation.} 
We assume log loss as the evaluation metric and a small learning rate $\eta$.  
Early stopping decides between the ensemble $F_0=\text{Tree}_0$ and the updated ensemble $F_1=\text{Tree}_0+\eta \cdot \text{Tree}_1$.  
The validation loss is
\begin{align*}
L_i &= \text{LogLoss}(F_i, H_{val}) \\
    &= -\frac{1}{n_1+n_2} \sum_{(x_t,r_t) \in H_{val}} 
       \Big[ r_t \log F_i(x_t) + (1-r_t) \log (1- F_i(x_t)) \Big].
\end{align*}

and Tree$_1$ is kept iff $L_1<L_0$; otherwise boosting stops, outputting the prediction based on Tree$_0$.  

\subsection{Analysis on allocation probability}
%see if this is too long
We analyze how the random variability introduced by early stopping translates into an allocation probability and show that it aligns closely with Thompson Sampling (TS). The key quantity is the probability of retaining Tree$_1$: acceptance yields greedy exploitation, while rejection (reverting to Tree$_0$) yields uniform exploration. Thus, this probability fully determines the induced allocation rule.  

We show, through a sequence of equivalent conditions, that the event of accepting Tree$_1$ can ultimately be expressed in terms of the training difference $\Delta_{tr}$. This reformulation allows us to analyze its distribution under the null and connect early stopping directly to classical hypothesis testing, as summarized in the following proposition:

\begin{proposition}[Equivalence of early stopping conditions]
In the two-armed setting with balanced splits, the following events are equivalent:
\begin{enumerate}
    \item Tree$_1$ is accepted (i.e., it reduces validation log-loss);
    \item The training and validation differences agree in sign: 
    \[
    \operatorname{sign}(\Delta_{tr}) = \operatorname{sign}(\Delta_{val});
    \]
    \item The training difference $\Delta_{tr}$ lies in the interval $(0,2\Delta)$.
\end{enumerate}
\end{proposition}

Moreover, the statistic $\Delta_{tr}$ can be rescaled to follow an asymptotic Normal distribution, allowing early stopping to be viewed as a classical hypothesis test.

\begin{theorem}[Asymptotic distribution of the training difference]
Let the true arm means be $\mu_1$ and $\mu_2$, and consider the null hypothesis $H_0:\mu_1=\mu_2$. 
If the sample ratio $N_1/N_2$ is held fixed, then as $N_1,N_2 \to \infty$,
\[
\frac{\Delta_{tr} - \Delta}{\sigma_\Delta} \;\;\xrightarrow{d}\;\; \mathcal{N}(0,1),
\]
where 
\[
\sigma_\Delta^2 \;=\; \frac{\mu_1(1-\mu_1)}{N_1} \;+\; \frac{\mu_2(1-\mu_2)}{N_2}.
\]
\end{theorem}

\noindent
It is worth emphasizing that this result is not a trivial application of the central limit theorem. 
The normalization uses $\sigma_\Delta$, the variance of the full-sample difference $\Delta$, while the numerator involves both $\Delta_{tr}$ and $\Delta$, which are correlated. 
The key fact is that the asymptotic variance nevertheless coincides with that of $\Delta$ itself. 
This subtle property is what ensures that the early stopping rule induces the correct allocation probabilities.

\noindent
Theorem~1 implies that, asymptotically, the probability of rejecting Tree$_1$ (and reverting to Tree$_0$ with uniform exploration) is equal to the two-sided $p$-value of the hypothesis test $H_0:\mu_1=\mu_2$. 
This is illustrated in Figure~\ref{fig:ap_comparison} (right), where the two tails correspond to rejection regions. 
In this case, the chance of playing the worse arm (arm~2 in our example) is $1/2$, so the allocation probability to arm~2 is equivalent to the one-sided $p$-value (half of the corresponding two-sided $p$-value). 
By contrast, Thompson Sampling assigns probability to the worse arm according to the posterior probability that it is optimal, which corresponds to the $p$-value of an equivalent Bayesian test, as shown in Figure~\ref{fig:ap_comparison} (left) . 
The two methods differ only marginally, due to the prior (Bayesian vs.\ frequentist) formulation.

\begin{figure}[h]
    \centering
    \includegraphics[width=0.75\linewidth]{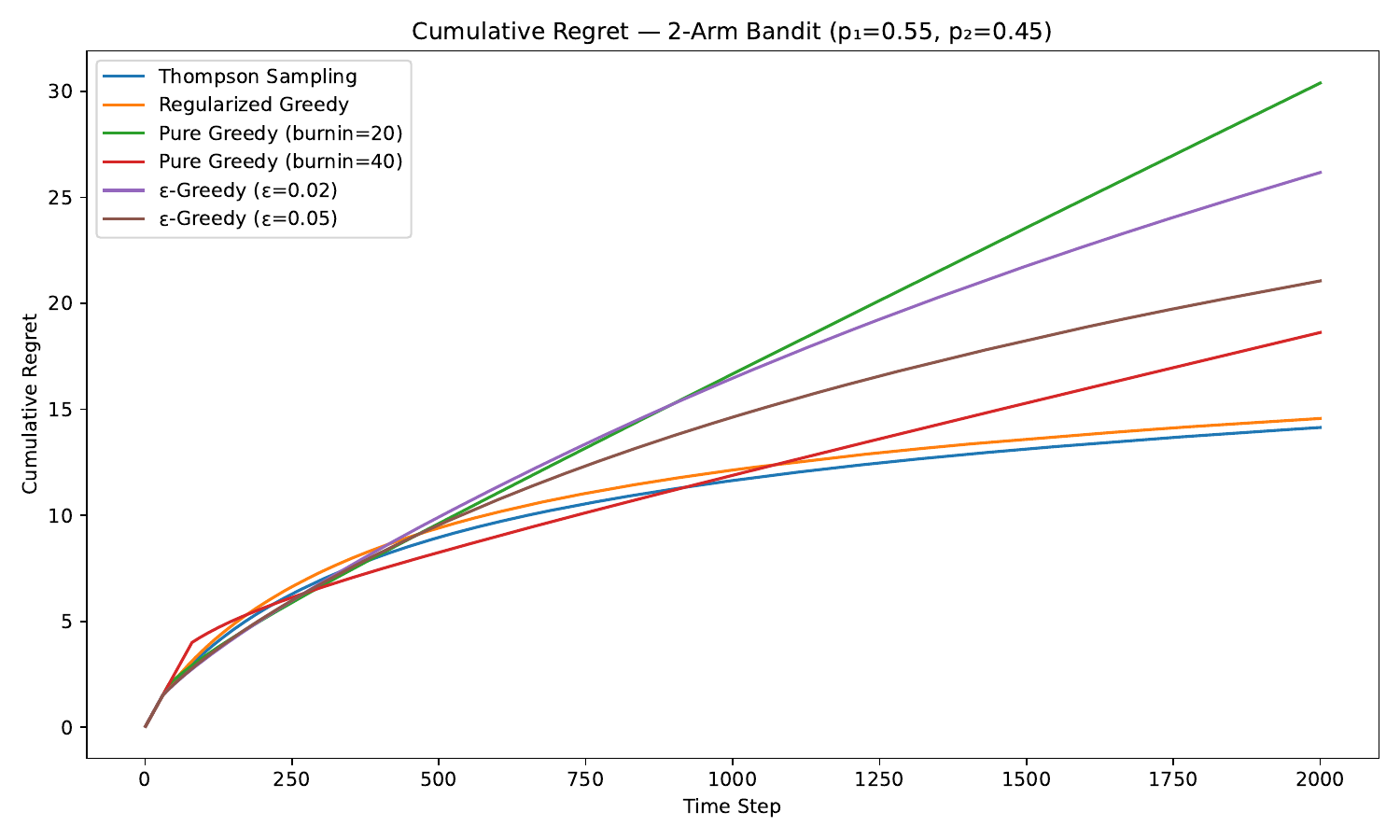}
    \caption{Cumulative Regret --- Stationary Two-Armed Bandit. Bernoulli arms with true means $0.55$ and $0.45$. Thompson Sampling uses a Beta$(0.1, 0.1)$ prior to avoid giving it an informational advantage over the regularized greedy, whose prior mean would otherwise coincide with the true arm means under a uniform prior.}
    \label{fig:reward_comp}
\end{figure}

\section{Data and Simulation Setup}

\subsection{Data Set Description}

\paragraph{Background.} 
In our simulation study, we construct a ground-truth reward function using a real-world business dataset derived from a large-scale digital marketing email promotion campaign. The objective of this campaign is to increase the company's Net Present Value (NPV) by sending targeted offers that encourage customers to extend their service contracts. Each record in the dataset represents a decision instance in which a personalized offer was sent to a customer, and the subsequent conversion outcome was observed. The reward is defined as $1$ if the customer extended their service contract within three days after receiving the promotion email and $0$ otherwise. 
This is a sparse data with conversion rate $0.003$.

\paragraph{Action and Contextual Features.} 
The dataset contains approximately 200{,}000 decision instances collected between January and July 2025. Each instance includes 113 contextual features describing user history and engagement behavior,
%TODO: ([DON'T FORGET THIS:] check if we have environmental features and add here with examples),
such as recency and frequency of interactions, tenure, and customer type. The action space consists of 17 offer-related features (e.g., discount percentage, number of free months, credit value), which together define 50 unique offer compositions. Some action features are defined jointly to maintain realistic marketing constraints—for example, higher percentage discounts are paired with lower credit values to avoid overly generous offers.  
In the simulation study, the underlying contextual feature distribution $\mathcal{X}$ is induced by randomly sampling from all 200{,}000 instances in the dataset.

\subsection{Simulation Setup}

\paragraph{Ground-Truth Reward Function.} 
To generate the ground-truth reward function, we construct a tabular dataset and train a gradient-boosted classification tree model, denoted by $F_{\text{email}}$. For any context–action pair $(a, x)$, the model predicts a conversion probability $F_{\text{email}}(a, x) \in (0,1)$. The observed reward is then simulated as
\[
R \sim \text{Bernoulli}(p = F_{\text{email}}(a, x)).
\]
The observed conversion rate across all actions is approximately $0.3\%$, reflecting the sparsity typical in large-scale marketing datasets. To reduce computational cost while maintaining signal balance, we do random subsample to augment mean rewards.

We construct both a simplified setting and the full setting. In the simplified setting we randomly keep only 5 action features and 2 contextual features in our original dataset, and we augment mean reward in the training dataset to $0.3$. In the full stimulation setting, to make it closer to the actual scenario, we keep about $5\%$ of the zero-reward cases while keeping all ones, yielding a marginal expected reward of approximately $0.06$ in the resulting training set.
% TODO in the above you're using percentage on one place and 0.05 in another place. Choose one single format

\paragraph{Simulation Epoch Schedule.} 
The simulation begins with a burn-in phase of 1{,}000 contextual samples, where actions are uniformly drawn from the action bank $\mathcal{A}$. The burn-in is followed by 150 epochs (100 in the stationary setting), each containing 100 samples. The exploration algorithm and its estimator are updated once per epoch. This configuration represents a typical daily learning schedule, where models are retrained or updated on a daily basis in practice.
% You haven't defined stationary setting, yet you have it here. 

\paragraph{Non-Stationary Reward Construction.} 
To simulate reward drift (i.e., simulate changes in customer preferences), we train two ground-truth reward functions: $F_0 = F_{\text{email,init}}$ and $F_1 = F_{\text{email,end}}$. The function $F_{\text{email,end}}$ corresponds to the ground truth reward model described earlier. To construct $F_{\text{email,init}}$, we introduce partial disagreement with $F_{\text{email,end}}$ while preserving the overall mean reward. Specifically, within the original training dataset, we randomly select a fraction $q$ of the positive-reward observations (reward $=1$), flip them to zero, and reassign the same number of positive rewards to randomly chosen negative samples (reward $=0$). This preserves the mean reward but generates a modified signal that leads to a different reward mapping. In our simulation we let $q=0.5$. Note that this doesn't inflict a huge reward function shift, since most rewards are still $0$, unchanged.
%We consider three levels of drift severity, with $q \in \{0.2, 0.5, 0.8\}$ representing small, medium, and large shifts, respectively.

\paragraph{Reward Change Schedule.} 
We simulate a gradual transition between $F_0$ and $F_1$ to mimic a non-stationary environment. For the first 1.5 months (epochs $t < 45$), rewards are generated by $F_0$. Over the next two weeks (epochs $45 \le t \le 60$), the reward function transitions linearly from $F_0$ to $F_1$, after which $F_1$ remains stable for the remainder of the simulation (approximately three months). Formally, the reward function at epoch $t$ is defined as
\[
F_t(a,x) =
\begin{cases}
F_0(a,x), & t < 45, \\
\frac{F_0(a,x)\,(60 - t) + F_1(a,x)\,(t - 45)}{15}, & 45 \le t \le 60, \\
F_1(a,x), & t > 60.
\end{cases}
\]

\subsection{Algorithms and Modifications}

\paragraph{Modification for Non-Stationary Scenarios.}
In the non-stationary simulation setting, we apply a rolling time-buffer mechanism to all algorithms, retaining only the most recent 1.5 months of data. This corresponds to keeping the latest 4{,}500 observations for model updates, allowing the learner to remain adaptive to recent trends while discarding outdated information. This design mirrors the data refresh cycle used in the real email campaign use case, where older data gradually loses relevance.

\paragraph{Algorithms Considered.}
We evaluate our proposed algorithm using its default parameter setting, $\alpha = 1$. This choice reflects the standard practice of performing early stopping with cross-validation without any explicit modification, allowing us to assess the intrinsic exploration behavior that arises naturally from regular model training.

For comparison, we include two additional baselines:  
(1) FALCON variant \citep{SimchiLeviXu2022}, (2) KL-EXP variant (i.e. the simple EXP algorithm) \citep{zhao2025sharp}, and
(3) $\varepsilon$-greedy, where $\varepsilon$ percent of actions are chosen uniformly at random from all available arms.
Both the original FALCON$+$ and the KL-EXP algorithm for general function classes relies on progressively increasing epoch lengths and discarding all past data at the end of each epoch. While this structure is theoretically advantageous for achieving asymptotic regret guarantees, it performs poorly in stationary cases without time windows (since older data are repeatedly discarded) and is fundamentally incompatible with the sliding-window framework used in non-stationary settings. More specifically, the sliding-window approach updates models continuously using a fixed-size buffer that gradually shifts over time. In contrast, FALCON$+$ and KL-EXP requires restarting training with increasingly longer epochs, causing updates to become less frequent and discrete over time.\\ 

Therefore, we adopt their core components. For Falcon, we take its allocation probability form that stem from the A/BW algorithm \citep{abe1999analysis}:
\[
P(a) = \frac{1}{K + \gamma(n)\,(\hat{r}_{\max} - \hat{r}_a)},
\]
where $K$ is the number of available actions, $\hat{r}_{\max}$ and $\hat{r}_a$ denote the estimated rewards for the best and current arms, respectively, and $\gamma(n)$ is a time-scaling factor. Following FALCON’s formulation, we set
\[
\gamma(n) = c \sqrt{K n},
\]
where $n$ is the number of samples in the current training buffer and $c$ is an algorithm parameter. Note that, we absorb the function class cardinality $|\mathcal{F}|$ into our $c$, since it is unknown (assuming it is actually finite) in our setting. We try a range of $c$ to set the best benchmark.

Similarly, for KL-EXP, we take its simple EXP form that has allocation probability: 
\[
P(a) \propto  \exp(\eta \cdot \hat{r}_a).
\]

% --- end commented-out section ---

\section{Result}\label{sec:result}
%TODO: check how to split and name differnt sections

\subsection{Simulation illustration and analysis of the bandit exploration behavior induced by early-stopping}
\begin{figure}[h]
    \centering
    \includegraphics[width=\linewidth]{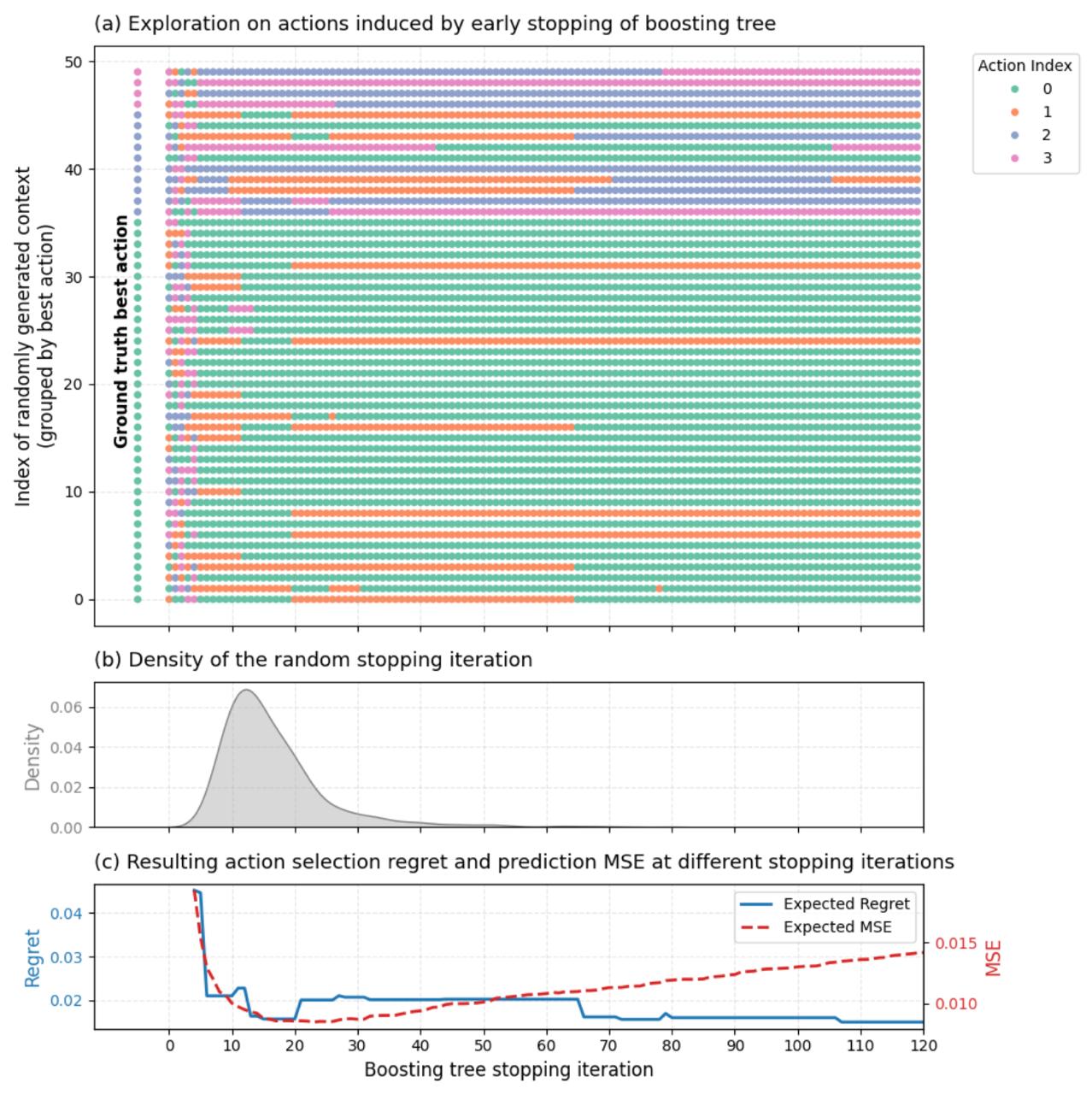}
    \caption{
    We train a gradient boosting tree under early-stopping (see Procedure~\ref{proc:early-stopping}) to 1{,}000 burn-in samples with contextual feature vectors drawn from the email promotion dataset. 
    By repeatedly running the algorithm, the resulting boosting model stops at different iterations due to stochasticity in the early-stopping process. 
    To illustrate its behavior and performance metrics, we show: 
    (a) the selected actions (indicated by color) when the boosting tree is stopped at different iterations; 
    (b) the distribution of stopping iterations observed during cross-validation early stopping; and 
    (c) the mean squared error (MSE) and regret associated with estimation and action selection at each iteration, evaluated using the ground-truth reward function.
    }
    \label{fig:iter_cv_illustration}
\end{figure}

Figure~\ref{fig:iter_cv_illustration} provides a comprehensive view of the dynamics that occur across boosting iterations and the resulting exploration behavior. We discuss the figure from several complementary perspectives.

First, by repeatedly applying our algorithm to the same data history, we observe variability in the resulting stopping iteration, driven by the randomness of the training–validation split. 
Figure~\ref{fig:iter_cv_illustration}(b) shows the empirical distribution of these stopping iterations. 
Figure~\ref{fig:iter_cv_illustration}(a) illustrates the actions that would be selected if the model were stopped at different iterations, evaluated over 50 randomly sampled contextual vectors. 
These contexts are grouped according to their best actions under the ground-truth reward function (shown at the left end of the plot). 
For clarity, we visualize only four actions, among which one (Action~1) is consistently suboptimal across all 50 contexts. 
The plot shows that stopping at different iterations indeed induces exploration across multiple arms. When Action~0 (shown in green) is the ground-truth best, the algorithm tends to exploit it consistently, with only limited exploration toward the orange arm—likely because Action~0 demonstrated a clear advantage during the burn-in phase, resulting in high confidence. In contrast, when Action~1 or~3 is the ground-truth best, the algorithm exhibits more active exploration between these two actions, as their estimated values are often similar. This behavior illustrates that the algorithm’s exploration primarily occurs among near-optimal actions, while maintaining strong exploitation when one action is clearly superior. Overall, the resulting selection patterns closely align with the ground-truth optimal actions.

Second, Figure~\ref{fig:iter_cv_illustration}(c) shows the mean squared error (MSE) and regret when the full model is truncated at various iterations. 
Combining panels (b) and (c), several observations emerge. 

\paragraph{Distinction and Relation between MSE and Regret.} 
The stopping locations are centered around the iteration that minimizes the ground-truth MSE, which aligns with standard machine-learning intuition. 
However, the regret curve behaves differently: regret is not minimized at exactly the same iteration as the MSE, and even when the model continues training far beyond the MSE optimum—incurring substantial overfitting—the regret changes only slightly and may even decrease marginally. 
This illustrates the distinction between minimizing MSE (an estimation objective) and minimizing instantaneous regret (a decision-making objective). 
Importantly, this distinction does not contradict our main argument that the early-stopping mechanism serves as a natural source of exploration for reducing long-term regret.  

To draw an analogy with the two-armed bandit setting, early stopping provides a favorable exploration–exploitation trade-off but does not maximize the immediate expected reward. 
Maximizing instant reward would correspond to always selecting the arm with the highest estimated mean—analogous to continuing training until full overfitting occurs.

Figure~\ref{fig:iter_cv_illustration}(c) also suggests that an over-fitted estimator can serve as a more faithful representation of a purely greedy policy with no exploration. 
This perspective is often overlooked in the literature, where the learning subroutine is delegated to an oracle and implicitly assumed to follow standard machine-learning training practices (e.g., with cross-validation and early stopping). 
In our subsequent simulations, we therefore adopt the fully over-fitted estimator as the pure-greedy benchmark and show that it indeed achieves higher short-term reward when only a brief time buffer is applied, compared to a cross-validated estimator trained with fewer iterations—a behavior consistent with the expected tendency of greedy algorithms.

% \paragraph{Insight for Algorithm Parameter Design.} 
% Since overfitting can sometimes yield higher instantaneous rewards, we design our algorithm to include a multiplier on the stopping iteration. 
% Specifically, the stopping point determined by early stopping is scaled by a factor $\alpha > 1$, allowing the model to train longer and emphasize exploitation when appropriate. 
% In environments where the reward function changes rapidly over time—requiring quicker adaptation and stronger exploitation of recently learned patterns—a larger multiplier (e.g., $\alpha = 1.5$ or $2$) is preferred.

\subsection{Algorithm empirical evaluation}

\begin{figure}[h]
    \centering
    \includegraphics[width=\linewidth]{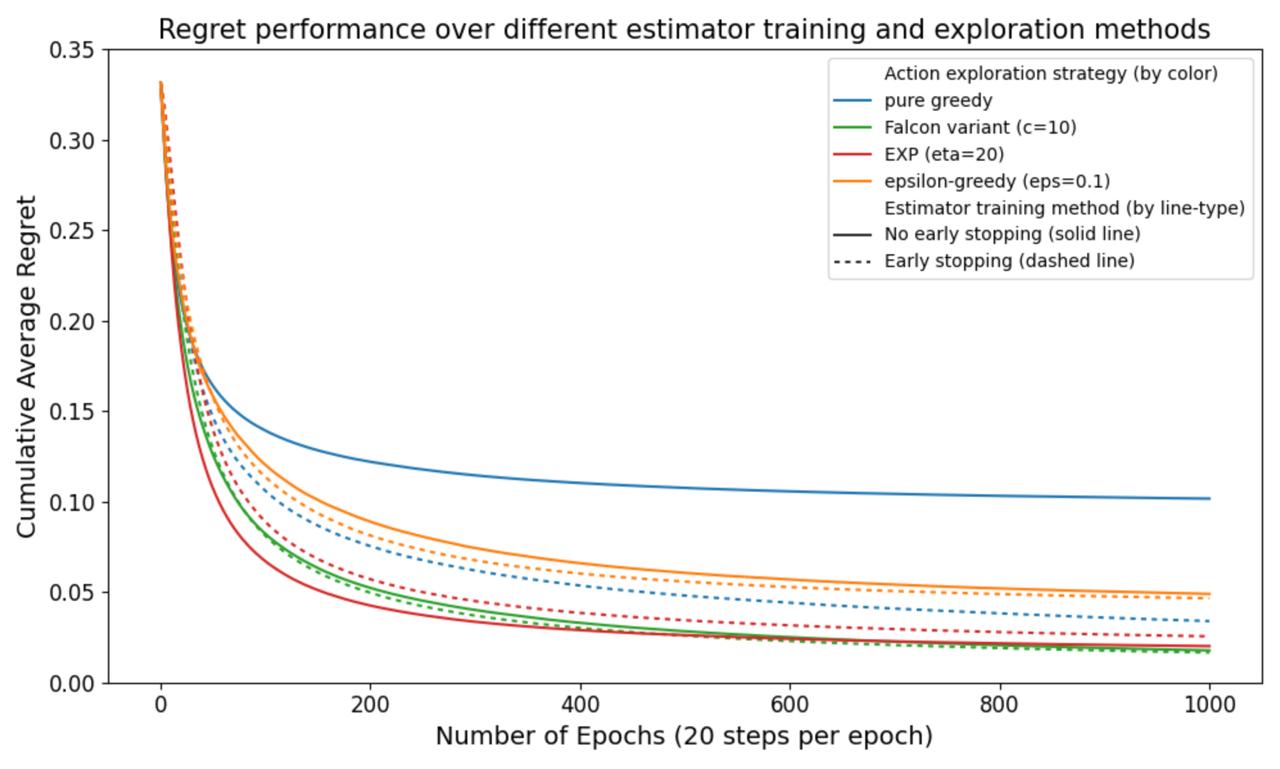}
    \caption{
    Regret performance in a simple stationary setting with only 5 action features (18 combinations) and 2 contextual features. We compare estimators trained using standard regularization approach (early stopping via cross validation, stops after 15-25 iterations) with estimators trained with fixed 30 iterations. And we combine additional exploration strategies upon them. The parameters for Falcon variant and EXP are optimized using a grid search.
    }
    \label{fig:regret_simple}
\end{figure}

\begin{figure}[h]
    \centering
    \includegraphics[width=\linewidth]{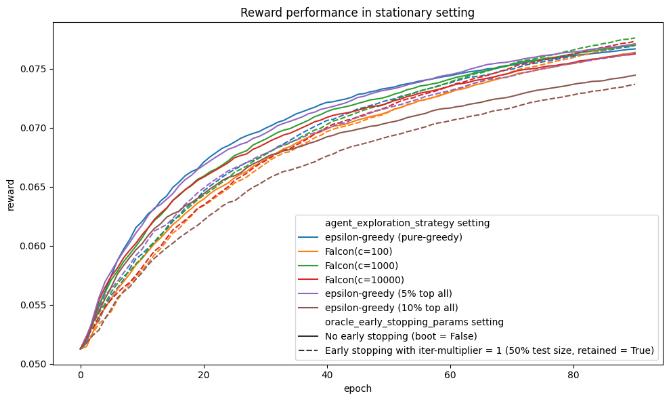}
    \caption{
    Reward performance in the stationary setting. All algorithms perform similarly, suggesting that little explicit exploration is needed, as the diversity of contextual features already induces sufficient implicit exploration across actions.
    }
    \label{fig:reward_st}
\end{figure}

%TODO: edit plot as Victor suggested. Also, try to estiamte the amount of exploration needed for EXP or Falcon
\begin{figure}[h]
    \centering
    \includegraphics[width=\linewidth]{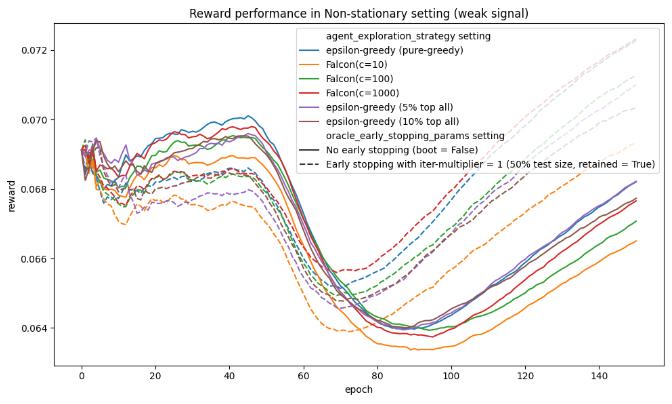}
    \caption{
    Reward performance in the non-stationary setting. The early-stopping–based oracle adapts more quickly to changes in the reward distribution, whereas additional explicit exploration provides limited or no improvement.
    }
    \label{fig:reward_nonsta}
\end{figure}

\begin{figure}[h]
    \centering
    \includegraphics[width=\linewidth]{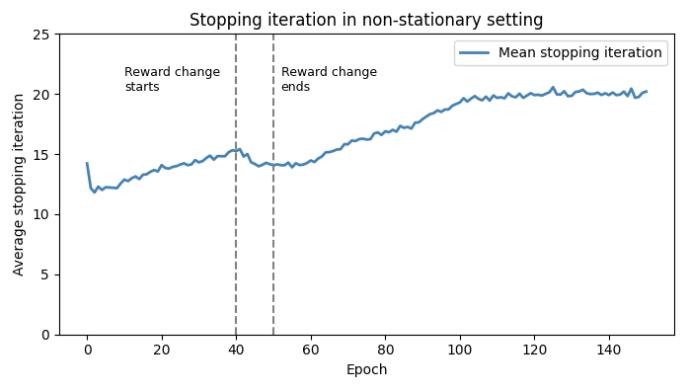}
    \caption{
    Average stopping iteration of the early-stopping oracle across different epochs. 
    The dashed vertical lines indicate the start and end of the reward distribution change.
    }
    \label{fig:stop_iter}
\end{figure}

% Two main insights emerge from the results: the estimator training method (determining their number of iterations) crucially shapes exploration behavior, and additional exploration such as $\varepsilon$-greedy offers no further gains.
We start by looking into stationary scenarios. We compare the combinations of three meta-exploration methods under two estimator training procedures.
The three meta-exploration methods are: (1) FALCON-based exploration, (2) $\varepsilon$-uniform exploration, and (3) no additional exploration (pure greedy).
The estimators are either trained with early stopping and retraining (without an iteration multiplier) or an unregularized models, representing fitting processes that more greedily exploit the data.

We first evaluate different estimator fitting processes and bandit algorithms in a simplified contextual setting (with 5 action features and 2 contextual features). We see in Figure~\ref{fig:regret_simple} that the unregularized estimator with a pure-greedy selection strategy performs poorly. However, the performance improves significantly when we either use a regularized estimator or add EXP or Falcon-like exploration. We acknowledge that if we manually search and optimize the algorithm parameters, both Falcon and EXP can achieve better reward performance than purely using the regularized estimation with pure-greedy selection in this simple contextual bandit case. Nevertheless, as we will show in more complex cases with additional contextual features (closer to our real business settings), such distinctions between using regularized estimator along and adding additional exploration strategy can become negligible. On the other hand, for naive exploration strategies such as epsilon-greedy, although they attempt to add more exploration to improve long-term rewards, their exploration is not efficient and thus can backfire on the reward. Note that this can also happen if we do not set the parameters for Falcon or EXP properly and end up with too much exploration. In other words, the regularized estimator already provides a good amount of exploration, and adding too much additional exploration can be counterproductive. Even in cases where the regularied estimate with greedy performs as well as alternative bandit algorithms, then the proposed approach is still desirable because it does so without needing any additional exploration strategy; sufficient exploration is a consequence of the early-stopping in the regularization estimator.

Then we go to the scenario where we use all 27 action features and 113 contextual features to construct the simulation reward function. We first consider the stationary scenario where the reward truth doesn't change over time. As shown in Figure~\ref{fig:reward_st}, the reward performance for different oracles and algorithms does not show a significant advantage over the pure greedy without early stopping (i.e., the algorithm with no exploration at all).
In addition, in early stages, the pure greedy without early stopping can perform best. This aligns with the composition of two analysis conclusions:

(1) The richness of contextual features gives all actions a good chance to be sampled, thereby introducing strong passive exploration (as suggested by \citep{bastani2021mostly,bietti2021bakeoff}); (2) It aligns with our analysis in Figure~\ref{fig:iter_cv_illustration} that overfitting can have more reward advantage than the estimator trained under usual ML standards.

% This also highlights the importance of our algorithm parameter $\alpha$, which provides an additional degree of exploitation that cannot be achieved otherwise (since previous work did not discuss overfitted estimators and offered no control over this aspect).

The fact that in scenarios with rich contextual features there is often no need for proactive exploration emphasizes the importance of further evaluating algorithms in the non-stationary case. In other words, the stationary case may not be a good indicator of whether an exploration strategy is effective.

This motivates the simulation on a non-stationary case (as shown in Figure~\ref{fig:reward_nonsta}). The results highlight several key findings. First, when the underlying reward function changes (around epoch~40), the early-stopping estimator adapts faster by detecting the shift more cautiously and exploring accordingly. This behavior is observed across all meta-algorithms and reflects the robustness of early-stopping regularization in non-stationary environments.

Second, we observe that adding additional meta-exploration on top of early-stopping estimators provides little to no improvement in cumulative reward, and in some cases even degrades performance.

Figure~\ref{fig:stop_iter} further illustrates the exploration behavior from the perspective of the oracle’s stopping iteration. When the underlying reward function starts to shift—introducing conflicting signals in the data buffer—we see that the average stopping iteration begins to decrease, meaning the oracle learns less deeply from the data and explores more. After the shift stabilizes, the iterations start to increase again. This illustrates the intrinsic exploration induced by early stopping.

Taken together, these results suggest that early stopping alone already induces a sufficient degree of exploration, making more complex meta-exploration strategies unnecessary in this context.

\section{Discussion}

%josh: put a main point for each 

%check: 
%compare the two idea:
%1. our contribution is 

%Existing state-of-art appraoch addresses contextual bandit problem in two steps: oracle training followed by exploration strategy applied on top of the trained estimator. In this work, we proposed a novel direction that combine them into one step. This is achieved by ultilizing the intisic Thompson Sampling liked exploration of the iteration based learner when using early stoppping as.
%Q: there seems to have too many conditions...

%In complex real world scenario, iteration based learneer, such as boosting tree and NN, are widely used. They provide a felxbilie way to approximte into reward turth by taking small steps, yet their balckbox nature

%CB without explicit algorithm for exploration and exploration one past could potentially lead to a good reward regret outcom. our work is related in X-Y. More boardly, our work point to a tech, you can do by X, you don't need to implicitly do exploration. Use existing model
%most paper need TS
%our work is better 
%by somethign 

%another point: limitation. what is the function of limitation:
%honest what the method does and doesn't (be precise
%good to get it accepted (if I don't but reviewer
%I can explain why this is not a reason to reject

%explain 
%

\textbf{Summary of our work:} 

In this work, we draw the link between machine learning regularization methods and bandit exploration strategy, and propose a simple and practical solution with wide applicability to real-world contextual bandit problems. We show that, by simply applying the pure-greedy action selection strategy over a complex model, such as a boosting tree, that is trained following standard machine learning routines with regularization, we can obtain a reasonable amount of exploration and good reward performance. As a result, for real world contextual bandit problems, practitioners can use pure-greedy action selection along, or add very light amount of EXP or Falcon liked exploration to achieve reliable performance. This offers new insight to the contextual bandit problem, and removes or reduces the challenge in selecting bandit algorithm and performing parameter selection.

More specifically, we note that when fitting an estimator iteratively and determining the stopping iteration based on the change in loss score on a validation set, the stochasticity in such a process intrinsically mimics Thompson Sampling behavior and provides appropriate exploration over actions. In other words, applying pure-greedy on top of such an estimator can exhibit exploration behavior that is fundamentally different from the common notion of “pure-greedy.” This is because the likelihood of observing improvement in the validation loss functions similarly to the confidence level of a hypothesis test on whether the newly learned iteration reflects a truly existing reward pattern. By accepting a new iteration and continuing such stochastic learning, we effectively mimic the logic of Thompson Sampling—taking randomized actions proportional to the likelihood that each action is optimal.

Our insight stems from a simplified two-armed, non-contextual setting. In that case, we show that when applying early stopping to a simplified gradient boosting tree (either stop at the root level and predict all arms the same, or learn the arm difference exhibited from data), the randomness in the stopping iteration and fitted estimators induces the exploration behavior of Thompson Sampling. Our theoretical analysis and empirical evaluation indicate that, when applying pure greedy action selection on top of such a fitted estimator, the resulting reward performance is almost identical to that of Thompson Sampling.

We further test our idea using a real-world dataset from an email marketing campaign sent to 331,336 users, with a total of 27 action features and 113 contextual features. We construct both stationary and non-stationary environments based on this data and compare our proposed approach—pure greedy on top of a boosting tree trained with early stopping—to pure greedy without early stopping. We also compare it to additional approaches that replace pure greedy with epsilon-greedy, FALCON variants and EXP. Our results suggest that the regularization process in estimator training provides a reasonable amount of exploration by itself. Applying additional exploration strategies can yield modest improvements in simplified settings when parameters are manually tuned through grid search. However, these gains become negligible as the scenario evolves toward more complex, realistic business environments with a large number of features.

%A more detailed discussion on result. But I think it is enough to discuss it in the Result section
%In the stationary setting, all approaches perform similarly; however, in the non-stationary setting, our proposed method adapts more quickly to reward changes compared to models trained without early stopping. On the other hand, adding explicit exploration strategies yields little or no additional improvement in reward performance. Our findings in the stationary case align with existing theory: the rich variety of contextual features can induce sufficient passive exploration, making active exploration unnecessary \citep{...}. Yet it also serves as evidence that overfitting and higher least square loss itself is not the root cause of bad reward performance (in our later no-stationary case).  In the non-stationary case, where active exploration is vital, our results show that early stopping introduces appropriate intrinsic exploration by itself.  

Taken together, our results suggest that estimators trained with early stopping can serve as an intrinsic exploration mechanism.  This helps with reducing the research and deployment cost and/or simplifying the implementation pipeline. More broadly, our work suggests the loss reduction problem and bandit exploration problem have an intrinsic relation, and opens an interesting direction for investigating more efficient and practical bandit solutions.

\vspace{1em}
\noindent\textbf{Recommendation for practitioners:} \\
As our results suggest, the estimator training process itself can induce a substantial degree of exploration, and often the improvement from additional exploration strategies can be small or even negative (especially when the number of contextual features is large). Practically, this means people can focus on the implementation of the estimator itself and view it not just as an estimator but also as a randomized exploration strategy. If they want to add additional exploration strategies, our results suggest that they start by considering only a small amount of such exploration. For example, they can start by setting the contextual bandit algorithm parameter  that has total allocation probability to sub-optimal arms less than $5\%$ or $2\%$.

% %%we consider oracle

% %inconterast to TS , UCB, we explore under which condition a pure greedy could work 

% %say how it stop and how it works
% %if you use this way to fit model 
% [TODO: remind them what alpha is, and see if more recs]
% In our current evaluation, we examined several choices of our algorithm parameter $\alpha$.  
% For practitioners, we recommend setting $\alpha = 1$ as a safe and standard option that best represents a Thompson-Sampling-like exploration strategy.  
% Alternatively, choosing $\alpha = 1.5$ can promote stronger exploitation. Especially in settings where the number of contextual features is large—which introduces additional exploration beyond what early stopping generates—a more aggressive choice of $\alpha$ in the direction of less exploration may be appropriate. Moreover, in scenarios where the reward information is weak in the data (either because the reward is sparse or the amount of data is limited) and the environment changes rapidly, setting $\alpha = 1.5$ can help make better use of the available data, ensuring competitive performance before the learned pattern becomes obsolete.

\vspace{1em}
\noindent\textbf{Limitations and future directions} 

At the same time, several limitations and opportunities for future research remain.

Our theoretical analysis is restricted to the two-armed case. While this setting reveals the essential connection between early stopping and exploration, extending the theoretical framework to general multi-armed and contextual environments is an important next step toward establishing formal regret guarantees.

%challenges: not sure how the choice of 
%give analysis on a formal cross validation setting: 
%what is similar v.s. what is not
%the fact it has a binary  hypotehis testing informed action  is similar
%however, the fact it update many features at once makes it diffnert. indtead of accpting each using a TS bahevior, it acctpi them at bulk, and thus it can easily passs the test if most are exist but some are not, and the p-value is a comunpendde one. 
%nevertheless, that fact early stopping is better than no-early stoppping + falcon or epsilon UR, and the fact overfitting doen't necessarily hrut reward (in stationer case and in analysis plot) both suggest that early stoping and regualrization has stong link to exporaiton behavior. 
%can train first the best MSE model, and train on the residual 

%our analysis is opeched by sevearl other thigns. 1. pure greedy and voer fitting can work good; 
%we tried it in limited setting. 

Beyond this, our results suggest a broader and more general principle: certain forms of regularization in machine learning—such as early stopping—inherently encode hypothesis-testing behavior, which in turn can induce exploration. Specifically, when performing early stopping, the outcome from a new iteration is accepted only if it reduces loss on the validation set. Validating it on the validation set is essentially performing a hypothesis test on whether the reward function pattern learned in the new iteration exists in reality. This process parallels Thompson Sampling, which exploits the learned arm differences (i.e., sampling the seemingly better arm) with the Bayesian posterior likelihood that such differences truly exist. 

In light of this, a promising direction is to further discuss and characterize such a property on different types of regularization. For instance, ridge regularization in linear regression primarily penalizes parameter magnitude and may not alter the relative ordering of actions, making it less suitable for inducing exploration.  
In contrast, parameters such as the minimum number of points per leaf in boosting trees, especially if scaled proportionally to the sample size, may behave analogously to hypothesis tests and may generalize the intrinsic exploration property identified here. 
Studying such connections can help better guide the use of pure-greedy action selection on different scenarios/models.

\bibliographystyle{plainnat}
\bibliography{references}

\clearpage   % or \newpage, forces a page break

\section*{Appendix}  % star removes the section number

\setcounter{section}{0}  % optional: reset counters if you use subsections below
\renewcommand{\thesubsection}{A.\arabic{subsection}}  % optional, label subsections as A.1, A.2, etc.

\subsection{Proof on Lemma 1}
\begin{proof}
Tree$_0$ predicts $\bar r_{tr}$. Tree$_1$ shifts predictions toward arm-specific means:
\[
p_i = \bar r_{tr} + \eta(\bar r_{tr,i}-\bar r_{tr}), \quad i=1,2.
\]
Denote the change on arm~1 as $\delta=\eta(\bar r_{tr,1}-\bar r_{tr})$. To keep the weighted mean fixed, arm~2 must shift by $-n_1/n_2\delta$, consistent with $p_2=\bar r_{tr}+\eta(\bar r_{tr,2}-\bar r_{tr})$.

Validation loss is
\[
L(\delta)=\tfrac{1}{n_1+n_2}\Big[n_1\ell(\bar r_{val,1},p_1)+n_2\ell(\bar r_{val,2},p_2)\Big],
\]
with $\ell(\nu,p)=-\nu\log p-(1-\nu)\log(1-p)$. Differentiating at $\delta=0$ (so $p_1=p_2=\bar r_{tr}$) yields
\[
\frac{dL}{d\delta}\Big|_{\delta=0} = \frac{n_1}{(n_1+n_2)\bar r_{tr}(1-\bar r_{tr})}\,(\bar r_{val,2}-\bar r_{val,1}).
\]

Thus the sign of $\delta$ is $\operatorname{sign}(\bar r_{tr,1}-\bar r_{tr,2})$, while the sign of $-dL/d\delta|_{\delta=0}$ is $\operatorname{sign}(\bar r_{val,1}-\bar r_{val,2})$. Loss decreases iff these agree, i.e.\ iff the training and validation differences have the same sign.
\end{proof}

\subsection{Proof of Lemma 2}

\begin{proof}
Let the empirical arm means on the full sample be $\bar r_1,\bar r_2$, and set $\Delta=\bar r_1-\bar r_2$. 
Assume $\Delta>0$ without loss of generality. 
Each arm $i$ contributes $N_i=2n_i$ samples, split evenly into training and validation sets of size $n_i$. 
Denote the training means by $\bar r_{tr,i}$ and validation means by $\bar r_{val,i}$, and define
\[
D_{tr} = \bar r_{tr,1}-\bar r_{tr,2}, \qquad 
D_{val} = \bar r_{val,1}-\bar r_{val,2}.
\]
Because the union of training and validation samples equals the full set, we have the exact identity
\[
\frac{D_{tr}+D_{val}}{2} = \Delta.
\]

The training and validation differences agree in sign precisely when
\[
0 < D_{tr} < 2\Delta.
\]
If $D_{tr}\leq 0$ then $D_{val}\geq 2\Delta$ and the signs differ, and likewise if $D_{tr}\geq 2\Delta$.

Each $\bar r_{tr,i}$ is a hypergeometric mean with expectation $\bar r_i$ and variance 
\[
\operatorname{Var}(\bar r_{tr,i}) = \frac{1}{2n_i}\bar r_i(1-\bar r_i)\Big(1-\frac{n_i}{2n_i}\Big) \approx \frac{1}{2n_i}\bar r_i(1-\bar r_i).
\]
Since the two arms are sampled independently, $D_{tr}$ is centered at $\Delta$ with variance 
\[
\operatorname{Var}(D_{tr}) \approx \frac{\bar r_1(1-\bar r_1)}{2n_1} + \frac{\bar r_2(1-\bar r_2)}{2n_2}.
\]
By the central limit theorem, $D_{tr}$ is approximately normal with this mean and variance. This aligns with the variance of $\bar r_{tr,i}$ when $r_1=r_2$.

% Rescaling $D_{tr}$ to unit variance shows that the event $0 < D_{tr} < 2\Delta$ corresponds exactly to the acceptance region of the two-sided Wald test for $H_0:\bar r_1=\bar r_2$ against $H_1:\bar r_1\neq \bar r_2$. 
% Thus, the probability of a sign match is the complement of the Wald rejection probability:
% \[
% \Pr\big(\operatorname{sign}(D_{tr})=\operatorname{sign}(D_{val})\big) 
% = \Pr\big(0<D_{tr}<2\Delta\big) 
% \approx 1-p_{\text{wald}}.
% \]
\end{proof}

%Claim:
%1. discorvery ponit people a differnt view /aspect of the contexutla bandit probel: regularied esitmatior itself has explorat behaveor
%don't need adddional (save resrarch cost
%save time/impelnt /reseach

%the contribution better exposes the strutuce of the problme : there's many soursec of vairablity that can geneatte exploration
%give peoplr better understaing of the problem . the cources natuurally poses 

%2.  in the context rich environment , it could replace the need for seraching a 
%they findthe evalutre, find best paramter for -> across settings
%presence of those tunning parameter creates a large conpute tmie and cost ->get multipled  worsen if you consider accors serting applications
% 3, in less, it reduces the need , make it easier to serach for the apraemter 
%narrow the scop of how much we would need to search

\end{document}